\documentclass[sts]{article}

\RequirePackage{fix-cm}
\RequirePackage[figuresright]{rotating}

\usepackage[utf8]{inputenc} 
\usepackage[T1]{fontenc}    

\usepackage[final,nonatbib]{nips_2017}
\usepackage{booktabs}       
\usepackage{amsfonts}       
\usepackage{nicefrac}       
\usepackage{microtype}      

\usepackage{url}             
\usepackage{hyperref}       

\usepackage{amsmath,amssymb,amsthm}
\usepackage[linesnumberedhidden, algonl,boxed,ruled,vlined]{algorithm2e}

\usepackage{graphicx}
\usepackage{caption}
\usepackage{subcaption}
\usepackage{wrapfig}
\usepackage{pgfplots}
\usepackage{tikz}
\usetikzlibrary{positioning,arrows}

\usepackage{bbm}
\usepackage{bm}
\usepackage{bibunits}
\defaultbibliographystyle{plain}
\usepackage{filecontents}

\newcommand{\Input}{\mathcal{X}}
\newcommand{\Output}{\mathcal{Y}}

\newcommand{\obs}{\mathbf x}

\newcommand{\Ind}[1]{\mathbbm{1}_{#1}}

\newcommand{\Dist}{\mathcal{D}}
\newcommand{\TrnSet}{\mathcal S}
\newcommand{\TstSet}{\mathcal T}

\newcommand{\mRb}{\texttt{DS}}

\newcommand{\OVA}{\texttt{OVA}}

\newcommand{\RecallTree}{\texttt{RecallTree}}
\newcommand{\fastxml}{\texttt{FastXML}}
\newcommand{\sleec}{\texttt{SLEEC}}
\newcommand{\leml}{\texttt{LEML}}

\newcommand{\pfastrexml}{\texttt{PfastReXML}}
\newcommand{\pdsparse}{\texttt{PD-Sparse}}
\newcommand{\MSVM}{\texttt{M-SVM}}
\newcommand{\bfZ}{\boldsymbol{z}}
\newcommand{\MaF}{\texttt{MaF}$_1$}

\newcommand{\covers}{{\cal K}}
\newcommand{\proba}{\mathbb{P}}

\newboolean{showcomments}
\setboolean{showcomments}{true}
\ifthenelse{\boolean{showcomments}}
  {\newcommand{\nb}[2]{\fbox{\bfseries\sffamily\scriptsize#1}{\sf\small$\blacktriangleright$\textit{#2}$\blacktriangleleft$}}}
      {\newcommand{\nb}[2]{}}

\newcommand{\graph}{\mathcal G}
\newcommand{\graphH}{\mathcal H}

\newcommand{\vertices}{\mathcal V}
\newcommand{\edges}{\mathcal E}
\newcommand{\Cset}{\mathcal C}
\newcommand{\Weight}{W}
\newcommand{\cover}{\mathcal C}
\newcommand{\Xset}{\mathcal X}
\newcommand{\Fset}{\mathcal H}
\newcommand{\realset}{\mathbb{R}}
\newcommand{\expectation}{\mathbb{E}}
\newcommand{\rademacher}{\mathfrak{R}}
\newcommand{\EmpLosss}[1]{\hat{R}_{#1}}

\newtheorem{definition}{Definition}
\newtheorem{theorem}{Theorem}
\newtheorem{lemma}{Lemma}

\newcommand{\prob}{\mathbb{P}}
\newcommand{\var}{\mathbb{V}}
\newcommand{\cY}{\mathcal{Y}}
\newcommand{\EmpRisk}[1]{{\tilde{R}}_{#1}}

\DeclareMathOperator*{\argmax}{argmax}
\usepackage[export]{adjustbox}

\usepackage{accents}

\newlength{\dhatheight}

\title{Aggressive Sampling for Multi-class to Binary Reduction with Applications to Text Classification}

\author{
  Bikash Joshi\\
  Laboratoire d'Informatics de Grenoble\\
  University of Grenoble Alpes\\ 
  Grenoble, France \\
  \texttt{firstname.lastname@imag.fr}\\
    \And
  Massih-Reza Amini\\
  Laboratoire d'Informatics de Grenoble\\
  University of Grenoble Alpes\\ 
  Grenoble, France\\
  \texttt{firstname.lastname@imag.fr}\\
    \And 
   Ioannis Partalas\\
   Expedia LPS \\ 
   Geneva, Switzerland\\
   \texttt{ipartalas@expedia.com}\\
   \And 
    Franck Iutzeler\\
    Laboratoire Jean Kuntzmann (LJK) \\
    University of Grenoble Alpes \\
    Grenoble, France \\
    \texttt{firstname.lastname@imag.fr}\\
    \And
  Yury Maximov\\
  Theoretical Division T-4 and CNLS\\
  Los Alamos National Laboratory\\
  Los Alamos, NM 87545 and \\
  Center for Energy Systems\\
  Skolkovo Institute of Science and Technology,\\
  Moscow, Russia\\
  \texttt{yury@lanl.gov} 
  \vspace{-6mm}
}

\begin{document}

\maketitle

\begin{abstract}
  We address the problem of multi-class classification in the case where the number of classes is very large. We propose a double sampling strategy on top of a multi-class to binary reduction strategy, which transforms the original multi-class problem into a binary classification problem over pairs of examples. The aim of the sampling strategy is to overcome the curse of long-tailed class distributions exhibited in majority  of  large-scale  multi-class classification problems and to reduce the number of pairs of examples in the expanded data.  We show that this strategy does not alter the consistency of the empirical risk minimization principle defined over the double sample reduction. Experiments are carried out on DMOZ and Wikipedia collections with 10,000 to 100,000 classes where we show the efficiency of the proposed approach in terms of training and prediction time, memory consumption, and predictive performance with respect to state-of-the-art approaches.
\end{abstract}

\section{Introduction}

Large-scale multi-class or extreme classification involves problems with extremely large number of classes as it appears in text repositories such as Wikipedia, Yahoo! Directory (\url{www.dir.yahoo.com}), or Directory Mozilla DMOZ (\url{www.dmoz.org}); and it  has recently evolved as a popular branch of machine learning with many applications in tagging, recommendation and  ranking.

The most common and popular baseline in this case is the one-versus-all approach (\OVA) ~\cite{Lorena08} where one independent binary classifier is learned per class. Despite its simplicity, this approach suffers from two main limitations; first, it becomes computationally intractable when the number of classes grow large, affecting at the same time the prediction. Second, it suffers from the class imbalance problem by construction.



Recently, two main approaches have been studied to cope with these limitations. The first one, broadly divided in  tree-based and embedding-based methods, have been proposed with the aim of reducing the effective space of labels in order to control the complexity of the learning problem.
Tree-based methods \cite{Beygelzimer09,bengio2010label,choromanska2013extreme,Choromanska14,daume2016logarithmic,prabhu2014fastxml,bhatia2015sparse, jasinska2016log} rely on binary tree structures where each leaf corresponds to a class and inference is performed by traversing the tree from top to bottom; a binary classifier being used at each node to determine the child node to develop. These methods have logarithmic time complexity with the drawback that it is a challenging task to find a balanced tree structure which can partition the class labels. Further, even though different heuristics have been developed to address the unbalanced problem, these methods suffer from the drawback that they have to make several decisions prior to reaching a final category, which leads to error propagation and thus a decrease in accuracy. On the other hand, label embedding approaches \cite{NIPS2009_3824, bhatia2015sparse, Mineiro15}  first project the label-matrix into a low-dimensional linear subspace and then use an \OVA\ classifier. However, the low-rank assumption of the label-matrix is generally transgressed in the extreme multi-class classification setting, and these methods generally lead to high prediction error.


The second type of approaches aim at reducing the original multi-class problem into a binary one by first expanding the original training set using a projection of pairs of observations and classes into a low dimensional dyadic space, and then learning a single classifier to separate between pairs constituted with examples and their true classes and pairs constituted with  examples with other classes  \cite{Abe:2004, Weston11, JoshiAPRUG15}. 
Although prediction in the new representation space is relatively fast, the construction of the dyadic training observations is generally time consuming and prevails over  the training and prediction times.

\noindent\textbf{Contributions.} In this paper, we propose a scalable multi-class classification method based on an aggressive double sampling of the dyadic output prediction problem. Instead of computing all possible dyadic examples, our proposed approach consists first in drawing a new training set of much smaller size from the original one by oversampling the most small size classes and by sub-sampling  the  few  big  size classes in order to avoid the curse of long-tailed class distributions common in the majority of large-scale multi-class classification problems \cite{BabbarSIGKDD14}. The second goal is to reduce the number of constructed dyadic examples. Our reduction strategy brings inter-dependency between the pairs containing the same observation and its true class in the original training set. Thus, we derive new generalization bounds  using  local  fractional Rademacher complexity showing that even with a shift in the original class distribution and also the inter-dependency between the pairs of example, the empirical risk minimization principle over the transformation of the sampled training set remains consistent. We validate our approach by conducting a series of experiments on subsets of DMOZ and the Wikipedia collections with up to 100,000 target categories.


\section{A doubly-sampled multi-class to binary reduction strategy}
\label{sec:Model}

We address the problem of monolabel multi-class classification  defined on joint space $\Input\times\Output$ where $\Input\subseteq \mathbb{R}^d$ is the {\em input space} and $\Output=\{1,\ldots,K\}\doteq[K]$ the {\em output space}, made of $K$ classes. Elements
of $\Input\times\Output$ are denoted as $\obs^{y}=(x,y)$.  Furthermore, we assume the training set $\TrnSet=(\obs_i^{y_i})_{i=1}^m$ is made of $m$ i.i.d examples/class pairs distributed according to a fixed but unknown probability distribution $\Dist$, and we consider a class of predictor functions $\mathcal{G}=\{g: \Input \times \Output \rightarrow \mathbb{R} \}$.

We define the instantaneous loss for predictor $g\in\mathcal{G}$ on example $\obs^{y}$ as:
\begin{equation}
\label{eq:InstLoss}
e(g,\obs^{y})=\frac{1}{K-1}\sum_{y'\in\Output\backslash\{y\}}\Ind{g(\obs^{y})\leq g(\obs^{y'})},
\end{equation}
where $\Ind{\pi}$ is the indicator function equal to $1$ if the predicate $\pi$ is true and $0$ otherwise. Compared to the classical multi-class error, $e'(g,\obs^{y})=\Ind{y\neq \argmax_{y'\in\Output} g(\obs^{y'})}$, 
the loss of~\eqref{eq:InstLoss}  estimates the average number of classes, given any input data, that get a greater scoring  by $g$ than the correct class. The loss (\ref{eq:InstLoss}) is hence a \emph{ranking} criterion, and the multi-class \texttt{SVM} of~\cite{Weston98} and AdaBoost.MR~\cite{schapire99} optimize convex surrogate functions of this loss. It is also used in label ranking~\cite{hullermeier2007minimizing}. 
Our objective is to find a function $g\in\mathcal{G}$ with a small expected risk
\begin{equation}
\label{eq:GenLossMC}
R(g)=\mathbb{E}_{\obs^{y}\sim \Dist}\left[e(g,\obs^{y})\right],
\end{equation}
by minimizing the empirical error defined as the average number of training examples $\obs_i^{y_i}\in\TrnSet$ which, in mean, are scored lower than $\obs_i^{y'}$, for $y'\in \Output\backslash\{y_i\}$ :
\begin{equation}
\label{eq:EmpLossMC}
    \EmpRisk{m}(g,\TrnSet) = 
    \frac{1}{m}\sum_{i=1}^m e(g,\obs_i^{y_i})=\frac{1}{m(K-1)}\sum_{i=1}^m 
    \sum_{y'\in\Output\backslash\{y_i\}}\Ind{g(\obs_i^{y_i}) - g(\obs_i^{y'}) \leq 0} .
\end{equation}
\subsection{Binary reduction based on dyadic representations of examples and classes}
\label{sec:dyadic}

In this work, we consider prediction functions of the form $g = f \circ \phi$, where $\phi: \Input\times\Output \to \mathbb{R}^p $ is a projection of the input and the output space into a joint feature space of dimension $p$; and $f \in \mathcal{F} = \{ f :\mathbb{R}^p\to \mathbb{R}\}$ is a function  that measures the adequacy between an observation $\obs$ and a class $y$ using their corresponding representation $\phi(\obs^y)$. The projection function $\phi$ is application-dependent and it can either be learned \cite{Weston11}, or defined using some heuristics \cite{Volkovs:12, JoshiAPRUG15}.

Further, consider the following dyadic transformation 
\begin{equation}
\label{eq:transfo}
T(\TrnSet)=\left( \left\{ \begin{array}{ll} \left(\bfZ_{j}=\left(\phi(\obs_i^{k\vphantom{k+1}}),\phi(\obs_i^{y_i})\right), ~~~ \tilde{y}_{j} = -1 \right) & \text{if } k<y_i \\\left(\bfZ_{j}=\left(\phi(\obs_i^{y_i}),\phi(\obs_i^{k})\right), ~~~  \tilde{y}_{j} = +1 \right) & \text{elsewhere}     \end{array}\right.\right)_{j\doteq(i-1)(K-1)+k},
\end{equation}
where $j=(i-1)(K-1)+k$ with $i\in [m], k\in [K-1]$; that expands a $K$-class labeled set $\TrnSet$ of size $m$ into a binary labeled set $T(\TrnSet)$ of size $N=m(K-1)$ (e.g.  Figure \ref{fig:Transfo} over a toy problem). With the class of functions
\begin{equation}
\label{eq:ClassOfFunction}
\mathcal H=\{h: \mathbb{R}^p\times\mathbb{R}^p \rightarrow \mathbb{R};  (\phi(\obs^y),\phi(\obs^{y'})) \mapsto  f(\phi(\obs^y)) - f (\phi(\obs^{y'}))  , f\in\mathcal{F}\},
\end{equation}
the empirical loss (Eq.~\eqref{eq:EmpLossMC})  can be rewritten as~:
\begin{equation}
\label{eq:EmpLossMC2}
\EmpRisk{T(\TrnSet)}(h) =\frac{1}{N}\sum_{j=1}^N \Ind{\tilde{y}_j h(\bfZ_j)\leq 0}.
\end{equation}
Hence, the minimization of Eq.~\eqref{eq:EmpLossMC2} over the transformation $T(\TrnSet)$ of a training set $\TrnSet$
\begin{wrapfigure}[11]{r}{0.5\textwidth}
\vspace{-2mm}
\resizebox{0.5\textwidth}{!}{
    \begin{tikzpicture}

\node[] at (1.5,2.5) {$\mathcal{S}$};  

\draw[rounded corners=5pt]   (2,2) rectangle ++(4,1);
  
\node[] at (2.5,2.5) {$\mathbf{x}_1^{y_1}$};  
\node[] at (3.5,2.5) {$\mathbf{x}_2^{y_2}$};  
\node[] at (4.5,2.5) {$\mathbf{x}_3^{y_3}$};  
\node[] at (5.5,2.5) {$\mathbf{x}_4^{y_4}$};

\draw[thick,->] (4,2) -- (4,0.9);

\node[] at (4.5,1.5) {$T$};

\draw[rounded corners=5pt]   (-2,0.9) rectangle ++(12,-2.3);

\node[] at (0,0.5) {\footnotesize $(\mathbf{z}_1 = (\phi( \mathbf{x}_1^{y_1}),\phi(\mathbf{x}_1^{y_2})) , +1)$};  
\node[] at (4,0.5) {\footnotesize $(\mathbf{z}_2 = (\phi( \mathbf{x}_1^{y_1}),\phi(\mathbf{x}_1^{y_3})) , +1)$};  
\node[] at (8,0.5) {\footnotesize $(\mathbf{z}_3 = (\phi( \mathbf{x}_1^{y_1}),\phi(\mathbf{x}_1^{y_4})) , +1)$};

\node[] at (0,0.0) {\footnotesize $(\mathbf{z}_4 = (\phi( \mathbf{x}_2^{y_1}),\phi(\mathbf{x}_2^{y_2})) , -1)$};  
\node[] at (4,0.0) {\footnotesize $(\mathbf{z}_5 = (\phi( \mathbf{x}_2^{y_2}),\phi(\mathbf{x}_2^{y_3})) , +1)$};  
\node[] at (8,0.0) {\footnotesize $(\mathbf{z}_6 = (\phi( \mathbf{x}_2^{y_2}),\phi(\mathbf{x}_2^{y_4})) , +1)$};

\node[] at (0,-0.5) {\footnotesize $(\mathbf{z}_7 = (\phi( \mathbf{x}_3^{y_1}),\phi(\mathbf{x}_3^{y_3})) , -1)$};  
\node[] at (4,-0.5) {\footnotesize $(\mathbf{z}_8 = (\phi( \mathbf{x}_3^{y_2}),\phi(\mathbf{x}_3^{y_3})) , -1)$};  
\node[] at (8,-0.5) {\footnotesize $(\mathbf{z}_9 = (\phi( \mathbf{x}_3^{y_3}),\phi(\mathbf{x}_3^{y_4})) , +1)$};

\node[] at (0,-1.0) {\footnotesize $(\mathbf{z}_{10} = (\phi( \mathbf{x}_4^{y_1}),\phi(\mathbf{x}_4^{y_4})) , -1)$};  
\node[] at (4,-1.0) {\footnotesize $(\mathbf{z}_{11} = (\phi( \mathbf{x}_4^{y_2}),\phi(\mathbf{x}_4^{y_4})) , -1)$};  
\node[] at (8,-1.0) {\footnotesize $(\mathbf{z}_{12} = (\phi( \mathbf{x}_4^{y_3}),\phi(\mathbf{x}_4^{y_4})) , -1)$};  

\end{tikzpicture}
}
\caption{A toy example depicting the transformation $T$ (Eq. \eqref{eq:transfo}) applied to a training set $\TrnSet$ of size $m=4$ and $K=4$.}
\label{fig:Transfo}
\end{wrapfigure}
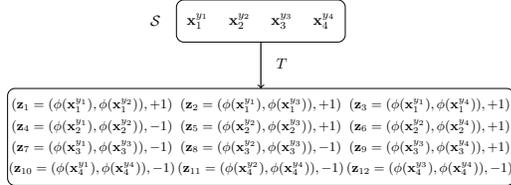
defines a binary classification over the pairs of dyadic examples. However, this binary problem takes as examples dependent random variables, as for each original example $\obs^y\in\TrnSet$, the $K-1$ pairs in $\{(\phi(\obs^y),\phi(\obs^{y'})); \tilde y\}\in T(\TrnSet)$ all depend on $\obs^y$.  In \cite{JoshiAPRUG15} this problem is studied by bounding the  generalization error associated to \eqref{eq:EmpLossMC2} using the fractional Rademacher complexity proposed in~\cite{UsunierAG05}. 
In this work, we derive a new generalization bounds based on Local Rademacher Complexities introduced in \cite{RalaivolaAmini15} that implies second-order (i.e. variance) information inducing faster convergence rates (Theorem \ref{thm:MulticlassRedBinaryGenBound}). Our analysis relies on the notion of graph covering introduced in \cite{Janson04RSA} and defined as~:
\begin{definition}[Exact proper fractional cover of $\graph$, \cite{Janson04RSA}]
\label{def:chromatic}
Let $\graph=(\vertices,\edges)$ be a
graph. $\cover=\{(\Cset_k,\omega_k)\}_{k\in[J]}$, for some positive integer $J$, with
$\Cset_k\subseteq\vertices$ and $\omega_k\in [0,1]$ is an exact proper
fractional cover of $\graph$, if:
i) it is {\em proper:} $\forall k,$ $\Cset_k$ is an {\em independent set}, i.e., there is no connections between vertices in~$\Cset_k$;
ii) it is an {\em exact fractional cover} of $G$: $\forall
  v\in\vertices,\;\sum_{k:v\in\Cset_k}\omega_k= 1$.
\end{definition}
The weight $\Weight(\cover)$ of $\cover$ is given by: $\Weight(\cover)\doteq\sum_{k\in[J]}\omega_k$ and the
minimum weight $\chi^*(\graph)=\min_{\cover\in\covers(\graph)} \Weight(\cover)$ over the set $\covers(\graph)$ of all exact proper fractional covers of $\graph$ is the {\em fractional chromatic number} of $\graph$.

From this statement, \cite{Janson04RSA} extended Hoeffding's inequality and proposed large deviation bounds for sums of dependent random variables which was the precursor of new generalisation bounds, including a  Talagrand’s type inequality for empirical processes in the dependent case presented in \cite{RalaivolaAmini15}.

With the classes of functions $\mathcal{G}$ and  ${\cal H}$ introduced previously, consider the parameterized family $\Fset_{r}$ which, for $r>0$, is defined as:
\[
    \Fset_{r} = 
    \{h:h\in\Fset,\mathbb{V}[h]\doteq\mathbb{V}_{\bfZ,\tilde{y}}[\Ind{\tilde{y}h(\bfZ)}]\leq r\},
\]
where $\mathbb{V}$ denotes the variance. 

\begin{wrapfigure}[15]{r}{0.5\textwidth}
\begin{flushleft}
    \vspace{-6mm}
    \hspace{-5mm}
    \qquad\includegraphics[width=70mm,height=40mm,left]{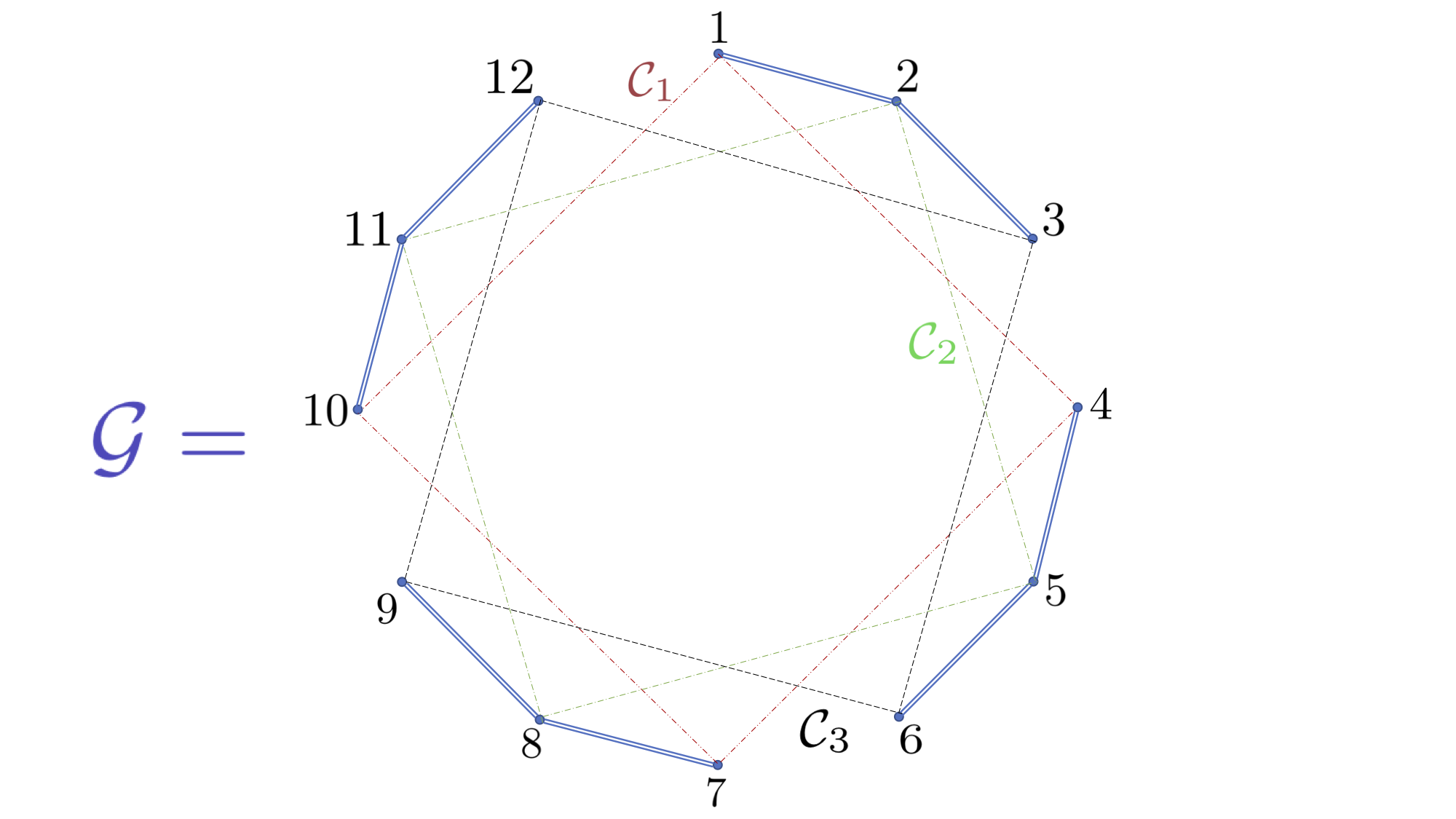}
\end{flushleft}
\vspace{-3mm}
\caption{The dependency graph $\graph=\{1, \ldots, 12\}$ corresponding to the toy problem of Figure \ref{fig:Transfo}, where dependent nodes are connected with vertices in blue double-line. The exact proper fractional cover $\Cset_1$, $\Cset_2$ and $\Cset_3$ is shown in dashed. The fractional chromatic number is in this case $\chi^*(\graph)=K-1=3$.}
\label{fig:DepGraph}
\end{wrapfigure}

The fractional Rademacher complexity introduced in \cite{UsunierAG05} entails our analysis~:
	    \[\rademacher_{T(\TrnSet)}(\mathcal{H})\hspace{-3pt}\doteq\hspace{-3pt}\frac{2}{N}
	    \expectation_{\xi}\hspace{-9pt}\sum_{k\in[K-1]}\hspace{-10pt}\omega_k\expectation_{\Cset_k}\hspace{-3pt}\sup_{h\in\mathcal{H}}\hspace{-5pt}\sum_{\alpha\in\Cset_k \atop \bfZ_\alpha \in T(\TrnSet)}\hspace{-6pt}\xi_\alpha h(\bfZ_\alpha),
	\]
	with $(\xi_i)_{i=1}^N$ a sequence of 
	independent Rademacher variables verifying 
	$\proba(\xi_n\!\!=\!\!1)=\proba(\xi_n\!\!=\!\!-1)=\frac{1}{2}$. If other is not specified explicitly we assume below all $\omega_k = 1$.
	
Our first result that bounds the generalization error of a function $h\in\cal H$; $R(h)=\mathbb{E}_{T(\TrnSet)} [\EmpRisk{T(\TrnSet)}(h)]$, with respect to its empirical error $\EmpRisk{T(\TrnSet)}(h)$ over a transformed training set, $T(\TrnSet)$, and the  fractional Rademacher complexity, $\rademacher_{T(\TrnSet)}(\mathcal{H})$, states as~:
	
\begin{theorem}
	\label{thm:MulticlassRedBinaryGenBound}
        Let $\TrnSet=(\obs_i^{y_i})_{i=1}^m\in (\Input\times \Output)^m$ be a dataset of $m$ examples drawn i.i.d. according to a probability distribution $\Dist$ over $\Input\times \Output$ and $T(\TrnSet)=((\bfZ_i,\tilde y_i))_{i=1}^N$ the transformed set obtained as in Eq.~\eqref{eq:transfo}. Then for any $1>\delta>0$ and $0/1$ loss $\ell:\{-1,+1\}\times\mathbb{R}\rightarrow [0,1]$, 
        with probability at least~$(1-\delta)$ the following generalization bound holds for all $h\in \mathcal \Fset_{r}$~:
        \vspace{-1mm}
        \[
            R(h)\leq 
                \EmpRisk{T(\TrnSet)}(h) + \rademacher_{T(\TrnSet)}(\ell\circ\Fset_r) 
                +\frac{5}{2}\left(\sqrt{\rademacher_{T(\TrnSet)}(\ell\circ\Fset_r)}
                +\sqrt{\frac{r}{2}}\right)\sqrt{\frac{\log \frac{1}{\delta}}{m}}
                +\frac{25}{48}\frac{\log \frac{1}{\delta}}{m}.
        \]
        \vspace{-2mm}
\end{theorem}
The proof is provided in the supplementary material, and it relies on the idea of splitting up the sum~\eqref{eq:EmpLossMC2} into several parts, each part being a sum of independent variables. 

\subsection{Aggressive Double Sampling}

Even-though the previous multi-class to binary transformation $T$ with a proper projection function $\phi$ allows to redefine the learning problem in a dyadic feature space of dimension $p \ll d$, the increased number of examples can lead to a large computational overhead. In order to cope with this problem, we propose a $(\pi, \kappa)$-\emph{double subsampling} of $T(\TrnSet)$, which first aims to balance the presence of classes by constructing a new training set $\TrnSet_\pi$ from $\TrnSet$ with probabilities $\pi=(\pi_k)_{k=1}^K$. 
\begin{wrapfigure}[20]{r}{0.69\textwidth}
\vspace{-3mm}\begin{algorithm}[H]
\KwIn{Labeled training set $\TrnSet=(\obs_i^{y_i})_{i=1}^m$}
{\bf initialization:} $\TrnSet_\pi \leftarrow \emptyset$\; 
$T_\kappa(\TrnSet_\pi)\leftarrow \emptyset$ \;
\For{$k=1..K$}{ 
 Draw randomly a set $\TrnSet_{\pi_k}$ of examples of class $k$ from $\TrnSet$ with  
 probability~$\pi_k$\;
 $\TrnSet_{\pi} \leftarrow \TrnSet_\pi \cup \TrnSet_{\pi_k}$\;  
 }
 \ForAll{$\obs^{y} \in \TrnSet_\pi $}{
    Draw uniformly a set $\Output_{\obs^y}$ of $\kappa$ classes from $\Output\backslash\{y\}$ \,  $\triangleright$  $\kappa \ll K$\;
    \ForAll{$k\in\Output_{\obs^y} $}{
        \If{$k<y$}{
            $T_\kappa(\TrnSet_\pi)\leftarrow T_\kappa(\TrnSet_\pi) \cup \left(\bfZ=\left(\phi(\obs^{k\vphantom{k+1}}),\phi(\obs^{y})\right)\!,~~~  \tilde{y} = -1 \right) $\;
        }
        \Else{
            $T_\kappa(\TrnSet_\pi)\leftarrow T_\kappa(\TrnSet_\pi) \cup \left(\bfZ=\left(\phi(\obs^{y}),\phi(\obs^{k})\right)\!, ~~~\tilde{y} = +1 \right) $\;
        }
    }
}
\KwRet{$T_\kappa(\TrnSet_\pi)$}
\label{algo}
\caption{$(\pi,\kappa)$-\mRb}
\end{algorithm}
\end{wrapfigure}
The idea here is to overcome the curse of long-tailed class distributions exhibited in majority of large-scale multi-class classification problems \cite{BabbarSIGKDD14} by oversampling the most small size classes and by subsampling the few big size classes in $\TrnSet$. The hyperparameters  $\pi$ are formally defined as $\forall k, \pi_k=P(\obs^y \in \TrnSet_\pi | \obs^y \in \TrnSet)$. In practice we set them inversely proportional to the size of each class in the original training set;  $\forall k, \pi_k\propto 1/\mu_k$ where $\mu_k$ is the proportion of class $k$ in $\TrnSet$. The second aim is to reduce the number of  dyadic examples  controlled by the hyperparameter $\kappa$. The pseudo-code of this \emph{aggressive double sampling} procedure, referred to as $(\pi,\kappa)$-\mRb, is depicted above and it is composed of two main steps.

\begin{enumerate}
    \item For each class $k\in\{1,\ldots,K\}$, draw randomly a set $\TrnSet_{\pi_k}$ of examples from $\TrnSet$ of that class  with  
    probability $\pi_k$, and let $\TrnSet_\pi=\displaystyle{\bigcup_{k=1}^{K}}\TrnSet_{\pi_k}$; 
    \item For each  example $\obs^{y}$ in $\TrnSet_\pi$,  draw uniformly $\kappa$ adversarial classes in $\Output\backslash\{y\}$. 
\end{enumerate}

After this double sampling, we apply the transformation $T$ defined as in Eq.~\eqref{eq:transfo}, leading to a set $T_\kappa(S_\pi)$ of size~$\kappa|\TrnSet_\pi|\ll N$. 

In Section~\ref{sec:Exps}, we will show that this procedure practically leads to dramatic improvements in terms of memory consumption, computational complexity, and a higher multi-class prediction accuracy when the number of classes is very large.  The empirical loss over the transformation of the new subsampled training set $\TrnSet_\pi$ of size $M$, outputted by the  $(\pi,\kappa)$-\mRb\ algorithm is~:
\begin{equation}
\label{eq:DEmpLossMCSampled}
    \EmpRisk{T_\kappa(\TrnSet_\pi)}(h) = \frac{1}{\kappa M} \sum_{({\tilde y}_\alpha, \bfZ_\alpha) \in T_\kappa(\TrnSet_\pi)} \Ind{{\tilde y_\alpha}h(z_\alpha) \le 0} = \frac{1}{\kappa M}\sum_{\obs^y \in \TrnSet_\pi}\sum_{y'\in \Output_{\obs^y}}
    \Ind{g(\obs^{y})- g(\obs^{y'}) \leq 0},
\end{equation}
which is essentially the same empirical risk as the one defined in Eq.~\eqref{eq:EmpLossMC} but taken with respect to the training set outputted by the $(\pi,\kappa)$-\mRb\ algorithm. Our main result is the following theorem which bounds the generalization error of a function $h\in\cal H$ learned by minimizing $\EmpRisk{T_\kappa(\TrnSet_\pi)}$.

\begin{theorem}
\label{thm:SlowRatesAp}
 Let $\TrnSet=(\obs_i^{y_i})_{i=1}^m\in (\Input\times \Output)^m$ be a training set of size $m$  i.i.d. according to a probability distribution $\Dist$ 
 over $\Input\times \Output$, and $T(\TrnSet)=((\bfZ_i,\tilde y_i))_{i=1}^N$
 the transformed set obtained with the transformation function $T$ defined as in Eq.~\eqref{eq:transfo}. Let $\TrnSet_\pi \subseteq \TrnSet$, $|\TrnSet_\pi| = M$, 
 be  a training set outputted by the algorithm $(\pi,\kappa)$-\mRb\ and $T(\TrnSet_\pi)\subseteq T(\TrnSet)$ its corresponding transformation. Then for any $1>\delta>0$ with probability at least $(1-\delta)$ the following risk bound holds for all $h\in {\cal H}$~:
    \[
        R(h) \leq 
            \alpha\EmpRisk{T_\kappa(\TrnSet_\pi)}(h) + 
            \alpha \rademacher_{T_\kappa(\TrnSet_\pi)}(\ell\circ \mathcal{H}) + 
            \alpha \sqrt{\frac{(K-1)\log \frac{2}{\delta}}{2M\kappa}} + 
            \sqrt{\frac{2 \alpha\log \frac{4K}{\delta}}{\beta(m-1)}} + 
            \frac{7\beta \log \frac{4K}{\delta}}{3(m-1)}.
    \]
    Furthermore, for all functions in the class ${\cal H}_r$, we have the following  generalization bound  that holds with probability at least $(1-\delta)$~:
    \begin{align*}
    R(h)\leq &
        \alpha\EmpRisk{T_\kappa(\TrnSet_\pi)}(h) + 
        \alpha \rademacher_{T_\kappa(\TrnSet_\pi)}(\ell\circ \mathcal{H}_r) + 
        \sqrt{\frac{2 \alpha\log \frac{4K}{\delta}}{\beta(m-1)}} + 
        \frac{7\beta \log \frac{4K}{\delta}}{3(m-1)} \\ 
        & + 
        \frac{5\alpha}{2}\left(\sqrt{\rademacher_{T_\kappa(\TrnSet_\pi)}(\ell\circ\Fset_r)} +   \sqrt{\frac{r}{2}}\right)\sqrt{\frac{(K-1)\log \frac{2}{\delta}}{M\kappa}} + 
        \frac{25\alpha}{48}\frac{\log \frac{2}{\delta}}{M}, 
    \end{align*}
where $\ell:\{-1,+1\}\times\mathbb{R}\rightarrow [0,1]$  $0/1$ is an instantaneous loss, and 
$\alpha  = \max_{y:\, 1\leq y \leq K} \eta_y/\pi_y$, $\beta = \max_{y:\, 1\leq y \leq K} 1/\pi_y$
and $\eta_y > 0$ is the proportion of class $y$ in $\TrnSet$.

\end{theorem}

The proof is provided in the supplementary material. This theorem hence paves the way for the consistency of the empirical risk minimization principle \cite[Th. 2.1, p. 38]{Vapnik1998} defined over the double sample  reduction strategy we propose.  



\subsection{Prediction with Candidate Selection} 

The prediction is carried out in the dyadic feature space, by first considering the pairs constituted by a test observation and all the classes, and then choosing the class that leads to the highest score by the learned classifier. 
\begin{wrapfigure}[9]{r}{0.64\textwidth}
\vspace{-4.7mm}
\begin{algorithm}[H]
\KwIn{Unlabeled test set $\TstSet$\;
Learned function $f^*:\mathbb{R}^p\rightarrow \mathbb{R}$\;}
{\bf initialization:}{ $\Omega \leftarrow \emptyset$}\;
 \ForAll{$\obs \in \TstSet $}{
    Select $\Output_\obs\subseteq\Output$ candidate set of $q$ nearest-centroid classes\;
    $\Omega \leftarrow \Omega \cup \argmax_{k \in \Output_\obs} f^*(\phi(\obs^k))$ \;
}
\KwRet{predicted classes $\Omega$}
\caption{Prediction with Candidate Selection Algorithm }\label{algo_pred} 
\end{algorithm}
\end{wrapfigure}
In the large scale scenario, computing the feature representations for all classes may require a huge amount of time. 
To overcome this problem we sample over classes by choosing just those that are the nearest to a test example, based on its distance with class centroids. 
Here we propose to consider class centroids as the average of vectors within that class. Note that class centroids are computed once in the preliminary projection of training examples and classes in the dyadic feature space  and thus represent no additional computation at this stage.  The algorithm above presents the pseudocode of this candidate based selection strategy.

%

\section{Experiments}
\label{sec:Exps}

In this section, we provide an empirical evaluation of the proposed reduction approach with the $(\pi,\kappa)$-{\mRb} sampling strategy for large-scale multi-class classification of document collections. First, we present the mapping $\phi:\Input\times\Output \rightarrow \mathbb{R}^p$.  Then, we provide a statistical and computational comparison of our method with state-of-the-art large-scale approaches on popular datasets.

\subsection{a Joint example/class representation for text classification } 
\label{features}

The particularity of text classification is that documents are represented in a vector space induced by the vocabulary of the corresponding collection \cite{Salton:1975:VSM:361219.361220}. Hence each class can be considered as a mega-document, constituted by the concatenation of all of the documents in the training set belonging to it, and simple feature mapping of examples and classes can be defined over their common words. Here we used $p=10$ features  inspired from learning to rank \cite{liu2007letor} by resembling a class and a document to respectively a document and a query (Table~\ref{tab:features}). All features except feature $9$, that is the distance of an example $x$ to the centroid of all examples of a particular class $y$, are classical. In addition to its predictive interest, the latter  is also used in prediction for performing candidate preselection.
\begin{table}[t!]
\begin{center}
\begin{tabular}{lllll}
\hline 
\multicolumn{5}{c}{Features in the joint example/class representation representation $\phi(\obs^y)$.} \\\hline 
\textbf{1.}~ $\displaystyle{\sum_{t\in y \cap \obs} \log(1+y_t)}$ &  ~~& \textbf{2.}~  $\displaystyle{\sum_{t\in y \cap \obs} \log\left(1+\frac{l_{\TrnSet}}{F_t}\right)}$ & ~~&
\textbf{3.}~  $\displaystyle \sum_{t\in y \cap \obs} I_t$  \\ 
\textbf{4.}~ $\displaystyle \sum_{t\in y \cap \obs} \frac{y_t}{|y|}.I_t$ & ~~&
\textbf{5.}~ $\displaystyle{\sum_{t\in y \cap \obs} \log\left(1+\frac{y_t}{|y|}\right)}$ & ~~& \textbf{6.}~ $\displaystyle{\sum_{t\in y \cap \obs} \log\left(1+\frac{y_t}{|y|}.I_t\right)}$ \\  \textbf{7.}~$\displaystyle{\sum_{t\in y \cap \obs} \log\left(1+\frac{y_t}{|y|}. \frac{l_{\TrnSet}}{F_t}\right)}$ &~~ &  \textbf{8.}~ $\displaystyle \sum_{t\in y \cap \obs} 1$  & ~~& \textbf{9.}~ $d(\obs^y, \texttt{centroid}(y))$ \vspace*{0.2cm} \\[0.3cm]
\multicolumn{5}{l}{\vspace*{0.2cm}
\textbf{10.}~ BM25 = {\small $\sum_{t\in y \cap x} I_{t}. \frac{2 \times y_{t}}{y_{t} + 
(0.25 + 
    0.75 \cdot \texttt{len}(y)/\texttt{avg}(\texttt{len}(y))}$}  } \\[0.2cm]
\hline
\end{tabular}
\end{center}
\caption{Joint example/class representation for text classification, where $t\in y \cap \obs$  are terms that are present in both the class $y$'s mega-document and document $\obs$. Denote by $\mathcal V$ the set of distinct terms within $\TrnSet$ then $\obs_t$ is the frequency of term $t$ in $\obs$, $y_t=\sum_{\obs\in y} \obs_t$, $|y|=\sum_{t\in \mathcal V}y_t$, $F_t=\sum_{\obs\in\TrnSet}\obs_t$, $l_{\TrnSet}=\sum_{t\in\mathcal V}\TrnSet_t$. Finally, $I_t$ is the inverse document frequency of term $t$, $len(y)$ is the length (number of terms) of documents in class $y$, and $avg(len(y))$ is the average of document lengths for all the classes}
\label{tab:features}
\vspace{-10mm}
\end{table}

\smallskip

Note that for other large-scale multi-class classification applications like recommendation with extremely large number of offer categories or image classification, a same kind of mapping can either be learned or defined using their characteristics  \cite{Volkovs:12, Weston11}.
\subsection{Experimental Setup} 

{\bf Datasets. }
We evaluate the proposed method using popular datasets from the Large Scale Hierarchical Text Classification challenge (\texttt{LSHTC}) 1 and 2  \cite{2015arXiv150308581P}. These datasets are provided in a pre-processed format using stop-word removal and stemming. Various characteristics of these datesets including the statistics of train, test and heldout are listed in Table~\ref{tab:data_statistics}. 
Since, the datasets used in LSHTC2 challenge were in multi-label format, we converted them to multi-class format by replicating the instances belonging to different class labels. Also, for the largest dataset (WIKI-large) used in LSHTC2 challenge, we used samples with 50,000 and 100,000 classes. The smaller dataset of LSHTC2 challenge is named as WIKI-Small, whereas the two 50K and 100K samples of large dataset are named as WIKI-50K and WIKI-100K in our result section.  

\begin{table}[!htb]
    \centering
    \begin{tabular}{c c c c c c}
        \hline
       Datasets  & \# of classes, $K$ & Train Size & Test Size & Heldout Size& Dimension, $d$ \\ \hline
       LSHTC1      & 12294  & 126871 & 31718 & 5000 & 409774 \\ 
       DMOZ        & 27875  & 381149 & 95288 & 34506& 594158 \\
       WIKI-Small  & 36504  & 796617 & 199155& 5000 & 380078\\ 
       WIKI-50K    &50000   & 1102754& 276939& 5000 & 951558 \\ 
       WIKI-100K   &100000  & 2195530& 550133& 5000 & 1271710\\ \hline
    \end{tabular}
    \caption{Characteristics of the datasets used in our experiments}
    \label{tab:data_statistics}
    \vspace{-7mm}
\end{table}


{\bf Baselines. }
We compare the proposed approach, denoted as the sampling strategy by $(\pi,\kappa)$-{\mRb}, with popular baselines listed below:
\vspace{-2mm}
\begin{itemize}
    \item $\OVA$: LibLinear \cite{Fan:2008} implementation of   one-vs-all SVM.
    \item $\MSVM$: LibLinear implementation of multi-class SVM proposed in \cite{Crammer:2002}.
    \item $\RecallTree$ \cite{daume2016logarithmic}: A recent tree based multi-class classifier implemented in Vowpal Wabbit.  
    \item $\fastxml$ \cite{prabhu2014fastxml}: An extreme multi-class classification method which performs partitioning in the feature space for faster prediction.
    \item $\pfastrexml$ \cite{jain2016extreme}: Tree ensemble based extreme classifier for multi-class and multilabel 
    problems. 
    \item $\pdsparse$ \cite{yen2016pd}: A recent approach which uses multi-class loss with $\ell_1$-regularization. 
\end{itemize}

Referring to the work \cite{yen2016pd}, we did not consider other recent methods $\sleec$ \cite{bhatia2015sparse} and $\leml$ \cite{yu2014large} in our experiments, since they have been shown to be consistently outperformed by the above mentioned state-of-the-art approaches.

  
    {\bf Platform.} In all of our experiments, we used a machine with an Intel Xeon 2.60GHz processor with 256 GB of RAM.

    {\bf Parameters.} Each of these methods require tuning of various hyper-parameters that influence their performance. For each methods, we tuned the hyperparameters over a heldout set and used the combination which gave best predictive performance. The list of used hyperparameters for the results we obtained are reported in the supplementary material (Appendix~B).
    
    {\bf Evaluation Measures.}
    Different approaches are evaluated over the test sets using accuracy and the macro F1 measure (\MaF), which is the harmonic average of macro precision and macro recall; higher \MaF thus corresponds to better performance. As opposed to accuracy, macro F1 measure is not affected by the class imbalance problem inherent to multi-class classification, and is commonly used as a robust measure for comparing predictive performance of classification methods.

\section{Results}
\label{sec:res}

\begin{table}[t!]
\centering
\scriptsize
 \begin{tabular}{|c|c|c|c|c|c|c|c|c|} 
 \hline
 Data &   & $\OVA$ & $\MSVM$  & $\RecallTree$ & $\fastxml$ & $\pfastrexml$ & $\pdsparse$ & $(\pi,\kappa)$-\mRb \\ 
 \hline\hline
 
 \textbf{LSHTC1} & train time  & 23056s & 48313s & 701s & 8564s  & 3912s  & 5105s  &  {321s}  \\
 m = 163589      & predict time& 328s & 314s &  {21s} & 339s   & 164s   & 67s     & 544s  \\ 
 d = 409774      & total memory& 40.3G & 40.3G &  {122M} & 470M   & 471M   & 10.5G  & 2G  \\ 
 K = 12294       & Accuracy       & 44.1\% & 36.4\%& 18.1\%& 39.3\% &39.8\%  & 45.7\%& 37.4\% \\ 
    &\MaF        & 27.4\% & 18.8\% & 3.8\% & 21.3\% & 22.4\% &  {27.7\%} & 26.5\%  \\ \hline
\textbf{DMOZ}   & train time   & 180361s & 212356s & 2212s & 14334s & 15492s  & 63286s &  {1060s}  \\
m = 510943      & predict time & 2797s   & 3981s &  {47s} & 424s   & 505s  & 482s   & 2122s  \\
d = 594158      & total memory & 131.9G  & 131.9G &  {256M} & 1339M  & 1242M   & 28.1G  & 5.3G   \\
K = 27875 & Accuracy       & 37.7\% &32.2\%  & 16.9\% & 33.4\%& 33.7\% &40.8\% &  27.8\%\\      
& \MaF          & 22.2\%  & 14.3\% & 1.75\% & 15.1\% & 15.9\% &  {22.7\%} & 20.5\% \\\hline
 
 \textbf{WIKI-Small} & train time    & 212438s & >4d & 1610s & 10646s & 21702s & 16309s &  {1290s}  \\
m = 1000772           & predict time & 2270s & NA &  {24s} & 453s  & 871s  & 382s   & 2577s  \\
d = 380078            & total memory & 109.1G & 109.1G &  {178M} & 949M  & 947M  & 12.4G  & 3.6G   \\
K = 36504            & Accuracy       & 15.6\% & NA & 7.9\% & 11.1\%& 12.1\%&15.6\%  & 21.5\% \\
& \MaF          & 8.8 \% & NA & <1\% & 4.6\% & 5.63\% & 9.91\% &  {13.3\%} \\ \hline

\textbf{WIKI-50K} & train time     & NA   & NA  & 4188s & 30459s & 48739s & 41091s &  {3723s}  \\
m = 1384693         & predict time & NA   & NA  &  {45s} & 1110s  & 2461s  & 790s   & 4083s  \\
d = 951558          & total memory & 330G & 330G&  {226M} & 1327M  & 1781M  & 35G    & 5G     \\
K = 50000           & Accuracy       & NA & NA & 17.9\% & 25.8\% & 27.3\% & 33.8\% & 33.4\%  \\
& \MaF          & NA   & NA  & 5.5\% & 14.6\% & 16.3\% & 23.4\% &  {24.5\%}  \\ \hline
 
\textbf{WIKI-100K} & train time     & NA   & NA   &  {8593s} & 42359s & 73371s & 155633s & 9264s  \\
m = 2750663          & predict time & NA   & NA   &  {90s} & 1687s  & 3210s  & 3121s   & 20324s \\
d = 1271710          & total memory & 1017G& 1017G&  {370M} & 2622M  & 2834M  & 40.3G   & 9.8G   \\
K = 100000           & Accuracy       & NA & NA & 8.4\% & 15\% & 16.1\% & 22.2\% & 25\% \\
&\MaF         & NA   & NA   & 1.4\% & 8\%    & 9\%    & 15.1\%  &  {17.8\%}  \\ \hline
 
 \hline
 \end{tabular}
 
 \caption{Comparison of the result of various baselines in terms of time, memory, accuracy, and macro~F1-measure}
\label{tab:res_comp}
\vspace{-7mm}
\end{table}

The parameters of the datasets along with the results for compared methods are shown in Table \ref{tab:res_comp}. The results are provided in terms of train and predict times, total memory usage, and predictive performance measured with accuracy and macro F1-measure (\MaF). For better visualization and comparison, we plot the same results as bar plots in Fig.~\ref{fig:res_comp} keeping only the best five methods while comparing the total runtime and memory usage.

First, we observe that the tree based approaches ($\fastxml$, $\pfastrexml$ and $\RecallTree$) have worse predictive performance compared to the other methods. This is due to the fact that the prediction error made at the top-level of the tree cannot be corrected at lower levels, also known as cascading effect. Even though they have lower runtime and memory usage, they suffer from this side effect.

For large scale collections {\small (\textbf{WIKI-Small}, \textbf{WIKI-50K} and \textbf{WIKI-100K})}, the solvers with competitive predictive performance are $\OVA$, $\MSVM$, $\pdsparse$ and $(\pi,\kappa)$-$\mRb$. However, standard $\OVA$ and $\MSVM$ have a complexity that grows linearly with $K$ thus the total runtime and memory usage explodes on those datasets, making them impossible. For instance, on Wiki large dataset sample of 100K classes, the memory consumption of both approaches exceeds the Terabyte and they take several days to complete the training. Furthermore, on this data set and the second largest Wikipedia collection {\small(\textbf{WIKI-50K} and \textbf{WIKI-100K})} the proposed approach is highly competitive in terms of Time, Total Memory and both performance measures comparatively to all the other approaches. These results suggest that the method least affected by long-tailed class distributions is $(\pi,\kappa)$-$\mRb$, mainly because of two reasons: first, the sampling tends to make the training set balanced and second, the reduced binary dataset contains similar number of positive and negative examples. Hence, for the proposed approach, there is an improvement in both accuracy and \MaF\ measures. 

The recent $\pdsparse$ method also enjoys a competitive predictive performance but it requires to store intermediary weight vectors during optimization which prevents it from scaling well. The $\pdsparse$ solver provides an option for hashing leading to fewer memory usage during training which we used in the experiments; however, the memory usage is still significantly high for large datasets and at the same time this option slows down the training process considerably.

In overall, among the methods with competitive predictive performance, $(\pi,\kappa)$-$\mRb$ seems to present the best runtime and memory usage; its runtime is even competitive with most of tree-based methods, leading it to provide the best compromise among the compared methods over the three time, memory and performance measures.

\begin{figure}[!t]
    \centering
    \includegraphics[width =\textwidth]{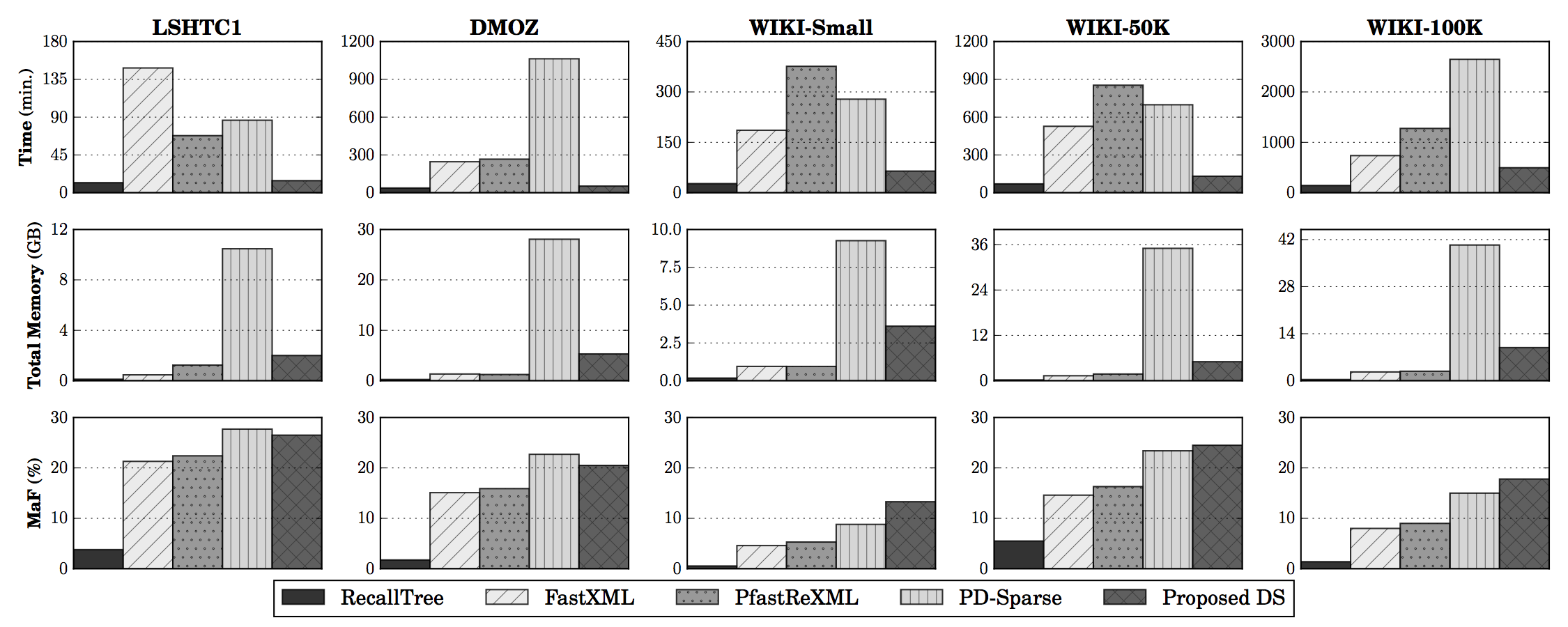}
    \caption{Comparisons in Total (Train and Test) Time (min.), Total Memory usage (GB), and \MaF\ of the five best performing methods on {\small \textbf{LSHTC1}, \textbf{DMOZ}, \textbf{WIKI-Small}, \textbf{WIKI-50K}} and {\small \textbf{WIKI-100K}}  datasets.}
    \label{fig:res_comp}
    \vspace{-8mm}
\end{figure}

\section{Conclusion}
\label{sec:conc}

We presented a new method for reducing a multiclass classification problem to binary classification. We employ similarity based feature representation for class and examples and a double sampling stochastic scheme for the reduction process. Even-though the sampling scheme shifts the distribution of classes and that the reduction of the original problem to a binary classification problem brings inter-dependency between the dyadic examples; we provide generalization error bounds suggesting that the Empirical Risk Minimization principle over the transformation of the sampled training set  still remains consistent.  Furthermore, the characteristics of the algorithm contribute for its excellent performance in terms of memory usage and total runtime and make  the proposed approach highly suitable for large class scenario.

\begin{bibunit}
{\small
\bibliographystyle{plain}   
\bibliography{m2binary_reduction}
}
\end{bibunit}

\newpage
\appendix
\newtheorem{innercustomthm}{Theorem}
\newtheorem*{theorem*}{Theorem}
\newenvironment{customthm}[1]
  {\renewcommand\theinnercustomthm{#1}\innercustomthm}
  {\endinnercustomthm}
  
  \section[Appendix A]{Theory Part}
\label{sec:app_theo}

%
%

\subsection{Technical Lemmas}

\bigskip 

\begin{lemma}
\label{lem:chromatic}
  Fractional chromatic number is monotone in graph inclusion: if $\graph = \langle \vertices_\graph, \edges_\graph\rangle \subseteq \graphH = \langle \vertices_\graphH, \edges_\graphH \rangle$ implies $\vertices_\graph \subseteq \vertices_\graphH$ and $\edges_\graph \subseteq \edges_\graphH$ we have $\chi^*(G) \le \chi^*(H)$.
\end{lemma}

\bigskip
 
\begin{proof}
Consider any exact proper fractional cover \cite{Janson04RSA} of graph $\graphH$, $\cover_\graphH = \{(\cover_k, \omega_k)\}_{k\in J}$ for some index set~$J$. By removing from each $\cover_k$ vertices that belong to $\vertices_\graphH \setminus \vertices_\graph$ and incident edges we get a cover $\cover_\graph = \{(\cover_k', \omega_k)\}_{k \in J}$ of graph $\graph$. Once for a certain $k$ holds  $\cover_{k}' = \varnothing$ we remove it from $\cover_\graphH$ which is essentially the same as assignment~$\omega_k \doteq 0$.

The cover $\cover_\graph$ is a proper fractional cover of $\graph$ since the number of connections between vertices in $\cover_k'$ is a subset of those in $\cover_k$ for any $k \in J$. The cover $\cover_\graph$ is also exact (modulo empty sets in $\cover_\graphH$) since for any $v$~:
\[
    v \in \vertices_\graphH \cap \vertices_\graph:\quad \sum_{k: v\in \cover_k'} \omega_k = \sum_{k: v\in \cover_k} \omega_k = 1,
\] 
where $\cover_\graphH = \{(\cover_k, \omega_k)\}_{k\in J}$ is an exact proper fractional cover of graph $\graphH$. That implies that each exact proper fractional cover $\cover_\graphH$ of graph $\graphH$ can be converted to an exact proper fractional cover $\cover_\graph$ of graph $\graph$ without increasing the covering cost $\Weight(\cover_\graph) \doteq \sum_{\cover_k \in \cover_\graph} \omega_k \le \Weight(\cover_\graphH)$. Denote the set of all exact proper fractional coverings of graph $\graph$ as $\covers(\graph)$ and coverings obtained by pruning $\covers(\graphH)$ as above through $\covers_\graphH(\graph)$.

By the definition of fractional chromatic number we have 
\begin{gather*}
\chi^*(\graph) {=} \min_{\cover \in \covers(\graph)} \Weight(\cover) \stackrel{(1)}{\leq} \min_{\cover \in \covers_\graphH(\graph)} \Weight(\cover) \leq \chi^*(\graphH),
\end{gather*}
where $(1)$ is implied by inclusion $\covers_\graphH(\graph) \subseteq \covers(\graph)$.
\end{proof}

\bigskip

\begin{lemma}[Empirical Bennet inequality, theorem 4 of \cite{maurer2009empirical}]
\label{lem:bennet}
Let $Z_1, Z_2, \dots, Z_n$ be i.i.d. variables with values in $[0,1]$ and let $\delta > 0$. Then with probability at least $1-\delta$ in $\mathbf{Z} = (Z_1, Z_2, \dots, Z_n)$ we have 
\[\expectation[\mathbf{Z}] - \frac{1}{n}\sum_{i=1}^n Z_i \le \sqrt{\frac{2 \var_n(\mathbf{Z}) \log 2/\delta}{n}} + \frac{7\log 2/\delta}{3(n-1)},\]
where $\var_n(\mathbf{Z})$ is the sample variance 
\[V_n(\mathbf{Z}) = \frac{1}{n(n-1)}\sum_{1\le i < j \le n} (Z_i - Z_j)^2.\]
\end{lemma}

%
%
%
%

\bigskip
 
 \begin{lemma}[Concentration of Fractionally Sub-Additive Functions, theorem 3 of \cite{RalaivolaAmini15}]
	\label{th:dependenttalagrand}
Let $\Fset$ be a set of functions from $\Xset$ to $\realset$ and assume all functions in $\Fset$ are measurable, square-integrable and 
satisfy~$\expectation[f(X_n)]=0,\forall n\in[N]$ and $\sup_{f\in\Fset}\|f\|_{\infty}\leq 1.$
Assume that $\cover=\{(\Cset_k,\omega_k)\}_{k \in J}$ is a cover of
the dependency graph of $X_{[N]}$ and let $\chi_f\doteq\sum_k{\omega_k}.$

Let us define: 
\[
Z\doteq\sum_{k\in[J]}\omega_k\sup_{f\in \Fset}\sum_{n\in\Cset_k}{f(X_n)}
\]
Let $\sigma_k$ be so that $\sigma^2_k= \sum_{n\in\Cset_k}\sup_{f\in\Fset}\expectation \left[f^2(X_n)\right]$,  $v\doteq\sum_k\omega_k\sigma_k^2+2\expectation [Z]$, and $c\doteq 25\chi_f/16$. Then, for any $t\geq 0$.
\begin{equation}
\proba\left(Z\geq \expectation [Z] + \sqrt{2cvt}+\frac{ct}{3}\right)\leq e^{-t}
\end{equation}
\end{lemma}

Let below $\Dist$ be a probability measure over $\Input\times \Output$. It can be decomposed into a direct product of $\Dist = \Dist_\cY \times \Dist_{\Input | \cY}$ with marginal distribution $\Dist_\cY$ over $\cY$ and conditional $\Dist_{\Input | \cY}$ over $\Input$. Let $\bar \Dist = {\bar \Dist}_\cY \times \Dist_{\Input | \cY}$ be a measure properly renormalized in accordance with the algorithm, e.g.~$\prob_{y \sim \bar\Dist}[y(\obs) = y] = \pi_y/\pi$, where $\pi = \sum_{i=y}^K \pi_y$. 

%

\bigskip

\begin{lemma}
\label{lem:main}
 Let $\TrnSet=(\obs_i^{y_i})_{i=1}^m\in (\Input\times \Output)^m$ be a dataset of $m$ examples drawn i.i.d. according to a probability measure $\Dist = \Dist_\cY \times \Dist_{\Input | \cY}$ over $\Input\times \Output$ and $T(\TrnSet)=((\bfZ_i,\tilde y_i))_{i=1}^N$ the transformed set obtained with the transformation function $T$ defined in (Eq.~(5)).  
 Let $\bar \Dist = \bar \Dist_\cY \times \Dist_{\Input | \cY}$, $\prob_{\bar \Dist_\cY}[y(\obs) = k] = \pi_k/\pi$, $1\leq i\leq K$, be a measure over $\Input\times \Output$ used in the $(\pi,\kappa)$-\mRb\ algorithm.
 %
 With the class of functions $\mathcal{G}=\{g: \Input \times \Output \rightarrow \mathbb{R} \}$ and $\mathcal H=\{h: h(\phi(\obs^y),\phi(\obs^{y'}))=g(\obs^{y})-g(\obs^{y'}), g\in\mathcal{G}\}$ for any $\delta > 0$ for all $h \in {\cal H}$ with probability at least $1-\delta$  we have~:
\[
    R(h) \leq 
        \alpha\expectation_{\obs^y \sim \bar \Dist}\EmpRisk{T(\obs^y)}(h) +
        \sqrt{\frac{2 \alpha\log 2K/\delta}{\beta(m-1)}} + 
        \frac{7\beta \log 2K/\delta}{3(m-1)}.
\]
holds  the for all $h\in \mathcal \mathcal{H}$, where $\ell:\{-1,+1\}\times\mathbb{R}\rightarrow [0,1]$ is the $0/1$ loss, and 
$\alpha  = \max_{y:\, 1\leq y \leq K} \pi\eta_y/\pi_y$, and $\beta = \max_{y:\, 1\leq y \leq K} \pi/\pi_y$, 
and $\eta_y > 0$ is the proportion the class $y$ in the training set $\TrnSet$.
\end{lemma}

\bigskip
 
\begin{proof}
First, decompose the expected risk $R(h)$ as a sum of the conditional risks over the classes
\begin{align}
\label{relAp:01}
    R(h) &=\expectation_{\obs^y \sim \mathcal{D}}[e_h(\obs^y)] \stackrel{(1)}{=} \expectation_{y \sim \mathcal{D}_\cY}\expectation_{\obs^y \sim \mathcal{D} | y(\obs) = y}[e_h(\obs^y) | y(\obs) = y]  
    \nonumber\\ 
    &\stackrel{(2)}{=}\sum_{y = 1}^K\prob_{\obs^y \sim \mathcal{D}} [y(\obs) = y] \cdot\expectation_{\obs^y \sim \mathcal{D} | y(\obs) = y} \left[e_h(\obs^y)| y(\obs) = y\right], 
\end{align}
where $(1)$ and $(2)$ are due to the law of total expectation. 


Similarly consider the expected loss $\expectation_{\obs^y \sim \bar{\cal D}}\EmpRisk{T(\obs^y)}(h)$~:
\begin{align}
\label{relAp:02}
\expectation_{\obs^y \sim \bar{\cal D}}\EmpRisk{T(\obs^y)}(h) &=\expectation_{\obs^y \sim \bar {\cal D}}[e_h(\obs^y)] 
\stackrel{(1)}{=} 
\expectation_{y \sim \bar{\cal D}_\cY}\expectation_{\obs^y \sim \mathcal{D}_{\Input|\Output}, y(\obs) = y}[e_h(\obs^y) | y(\obs) = y] 
    \nonumber\\ 
  & \stackrel{(2)}{=} \sum_{y = 1}^K\prob_{\obs^y \sim \bar{\cal D}} [y(\obs) = y] \cdot\expectation_{\obs^y \sim \mathcal{D} | y(\obs) = y} \left[e_h(\obs^y)| y(\obs) = y\right]  \nonumber\\
&=    \sum_{y = 1}^K \frac{\pi_y}{\pi} \cdot\expectation_{\obs^y \sim \mathcal{D} | y(\obs) = y} \left[e_h(\obs^y)| y(\obs) = y\right], 
\end{align}
where $(1)$ and $(2)$ are also due to the law of total expectation.


From \eqref{relAp:01} and \eqref{relAp:02} we conclude
\begin{equation}
\label{relAp:03}
    R(h) \leq \max_{y: \, 1\leq y\leq K} \frac{\prob_{\obs^y \sim {\cal D}}[y(\obs) = y]}{\pi_y/\pi} \cdot \expectation_{\obs^y \sim {\bar {\cal D}}} \EmpRisk{T(\obs^y)}(h)
\end{equation}

Finally, we need to bound the multiplier in front of $\expectation_{\obs^y \sim {\bar {\cal D}}} \EmpRisk{T_\kappa(\TrnSet_\pi)}(h)$ in Eq.~\eqref{relAp:03}. Denote through~$\eta_y$ an empirical probability of the class $y \in \cY$~:
\[\eta_y = \frac{1}{m}\sum_{\obs \in \TrnSet} \Ind{y(\obs) = y}.\]
Note, that empirical variance $\var_n(\eta_y)$ in accordance with lemma~\ref{lem:bennet} is 
\[\var_m(\eta_y) = \frac{\eta_y(1-\eta_y)m}{(m-1)}\]
For any $y\in \cY$ we have with probability at least $1-\delta/K$ by lemma~\ref{lem:bennet}~: 
\begin{eqnarray}
  \prob_{\obs^y \sim {\cal D}}[y(\obs) = y] \stackrel{}{\leq} \eta_y + \sqrt{\frac{2\var_m(\eta_y)\log 2K/\delta}{m}} + \frac{7\log 2K/\delta}{3(m-1)} \stackrel{(1)}{=} 
  \nonumber\\
  \eta_y + \sqrt{\frac{2\eta_y(1-\eta_y)\log 2K/\delta}{m-1}} + \frac{7\log 2K/\delta}{3(m-1)} 
  \stackrel{(2)}{\leq} \nonumber\\
  \eta_y + \sqrt{\frac{2\eta_y\log 2K/\delta}{m-1}} + \frac{7\log 2K/\delta}{3(m-1)},\nonumber
\end{eqnarray}
where $(1)$ is a substitution of $\var_m(\eta_y)$ by its explicit value; $(2)$ is due to the fact that $0 < \eta_y \leq 1$.

Then simultaneously for all $y\in \cY$ we have with probability at least $1-\delta$~:
\begin{equation*}
\prob_{\obs^y \sim {\cal D}}[y(\obs) = y] \leq \eta_y + \sqrt{\frac{2\eta_y\log 2K/\delta}{m-1}} + \frac{7\log 2K/\delta}{3(m-1)}
\end{equation*}

Thus with probability at least $1-\delta$~:
\begin{equation}
\label{relAp:04}
\max_{y: \, 1\leq y\leq K} \frac{\prob_{\obs^y \sim {\cal D}}[y(\obs) = y]}{\pi_y/\pi} \le \alpha + \sqrt{\frac{2 \alpha\log 2K/\delta}{\beta(m-1)}} + \frac{7\beta \log 2K/\delta}{3(m-1)},
\end{equation}
with 
\[\alpha  = \max_{y:\, 1\leq y \leq K} \frac{\pi\eta_y}{\pi_y}, \quad \beta = \max_{y:\, 1\leq y \leq K} \frac{\pi}{\pi_y}\]

From equations \eqref{relAp:03} and \eqref{relAp:04} and the fact that $\expectation_{\obs^y \sim {\bar {\cal D}}^N}\EmpRisk{T(\obs^y)}(h) \leq 1$, we have with probability at least $1-\delta$~:
\[
R(h) \leq \alpha\expectation_{\obs^y \sim \bar {\cal D}}\EmpRisk{T(\obs^y)}(h)  + \sqrt{\frac{2 \alpha\log 2K/\delta}{\beta(m-1)}} + \frac{7\beta \log 2K/\delta}{3(m-1)}.
\]
\end{proof}

\bigskip
 
\subsection{Proofs}
The results of the previous section entail hence the following lemma.

\bigskip

 \begin{lemma}
	\label{lem:MulticlassRedBinaryGenBound}
        Let $\TrnSet=(\obs_i^{y_i})_{i=1}^m\in (\Input\times \Output)^m$ be a dataset of $m$ examples drawn i.i.d. according to a probability distribution $\Dist$ over $\Input\times \Output$ and $T_\kappa(\TrnSet)=((\bfZ_i,\tilde y_i))_{i=1}^{m\kappa}$ the transformed set obtained as in Eq.~(5) and draw $\kappa$ adversarial samples by algorithm $(\pi,\kappa)$-\mRb. 
        With the class of functions $\mathcal{G}=\{g: \Input \times \Output \rightarrow \mathbb{R} \}$ and $\mathcal H=\{h: h(\phi(\obs^y),\phi(\obs^{y'}))=g(\obs^{y})-g(\obs^{y'}), g\in\mathcal{G}\}$, consider the parameterized family $\Fset_{r}$ which, for $r>0$, is defined as~:
        \[
            \Fset_{r}=\{h:h\in\Fset,\mathbb{V}[h]\doteq\mathbb{V}_{\bfZ,\tilde{y}}[\Ind{\tilde{y}h(\bfZ)}]\leq r\},
        \]
        where $\mathbb{V}$ denotes the variance. 
        Then for any $\delta>0$ and $0/1$ loss $\ell:\{-1,+1\}\times\mathbb{R}\rightarrow [0,1]$, 
        with probability at least~$(1-\delta)$ the following generalization bound holds for all $h\in \mathcal \Fset_{r}$~:
        \vspace{-1mm}
        \begin{multline*}
            R(h)\leq 
                \EmpRisk{T_\kappa(\TrnSet)}(h) + \rademacher_{T_\kappa(\TrnSet)}(\ell\circ\Fset_r) + \\ 
                \frac{5}{2}\left(\sqrt{\rademacher_{T_\kappa(\TrnSet)}(\ell\circ\Fset_r)} +
                \sqrt{\frac{r}{2}}\right)\sqrt{\frac{(K-1)\log 1/\delta}{m\kappa}} +
                \frac{25}{48}\frac{\log 1/\delta}{m}.
        \end{multline*}
        \vspace{-2mm}
\end{lemma}
\begin{proof}
Consider the function $\Phi$ defined as:
\[
    \Phi(\underline{X},r)\doteq 
        N\sup_{h\in\Fset_{r}}\left[\expectation_{\underline{X}'}[\EmpRisk{T(\underline{X}')}(h)]-\EmpRisk{T(\underline{X}')}(h)\right],
\]
\noindent where $\underline{X}'$ is an i.i.d. copy of $\underline{X}$  and where we have used the notation 
$
    \expectation_{\underline{X}'}[\EmpRisk{T(\underline{X}')}(h)]$ 
for 
$
\expectation_{T(\TrnSet)}\EmpLosss{N}(h,T(\TrnSet))
$
to make explicit the dependence on the sequence of dependent variables $\underline{X}'$.
It is easy to see that:

\begin{align}
	\Phi(\underline{X},r)&\leq \sum_{k\in[K-1]} \omega_k \sup_{h\in\Fset_{r}}\sum_{\alpha\in\Cset_k}\left[\expectation_{({\tilde y}',\bfZ')}[\Ind{\tilde{y'}h(\bfZ')}]-\Ind{\tilde{y}_\alpha h(\bfZ_{\alpha})}\right]
	\notag\\
	&= Z.\label{eq:phibound}
\end{align}

\noindent Lemma~\ref{th:dependenttalagrand} readily applies to upper bound the right hand side of~\eqref{eq:phibound}.
Therefore, for $t>0$, the following holds with probability at least $1-e^{-t}$:

\begin{equation*}
	\Phi(\underline{X},r) \leq\expectation [Z]+ \sqrt{2cvt}+\frac{ct}{3},
\end{equation*}

\noindent where $c=25\chi_f/16=25(K-1)/16$ and $v\leq Nr+2\expectation [Z]$.
Using $\sqrt{a+b}\leq\sqrt{a}+\sqrt{b}$ and $2\sqrt{ab}\leq u a +b/u$ for all $u>0$, we get, 
\begin{equation*}
    \forall u>0,	\Phi(\underline{X},r) \leq(1+u)\expectation [Z]+\sqrt{2cNrt}+\left(\frac{1}{3}+\frac{1}{u}\right)ct.
\end{equation*}
Furthermore, with a simple symmetrization argument, we have,

\[
\expectation[Z]=\expectation\left[\sum_{k\in[K-1]} \omega_k \sup_{h\in\Fset_{r}}\sum_{\alpha\in\Cset_k}
\left[\expectation_{{\tilde y}',\bfZ')}[\Ind{{\tilde y}'h(\bfZ')}]-\Ind{\tilde{y}_\alpha h(\bfZ_{\alpha})}\right]\right]\leq N\rademacher(\ell\circ\Fset_r),
\]
with $\omega_k = 1$ for all $k$ since the fractional chromatic number of the dependency graph corresponds to the sample $T(\TrnSet)$ equals to $K-1$ and stands for the covering determined by Eq.~(4) with unit weights~$\omega_k$.  

Further, as $N=m\kappa$, and fractional chromatic number of $T_\kappa(\TrnSet) \leq T(\TrnSet) = K-1$ (theorem 1 of \cite{JoshiAPRUG15}),  with probability at least~$1-e^{-t}$,  we have for all $h\in\Fset_{r}$
		
\begin{multline}
\label{eq:FirstGenBound}
	R(h)-\EmpRisk{T_\kappa(\TrnSet)}(h)  \leq \\
	 \inf_{u>0}\left( (1+u)\rademacher_{T_\kappa(\TrnSet)}(\ell\circ\Fset_{r})+\frac{5}{4}\sqrt{\frac{2(K-1)rt}{m\kappa}}+
	\frac{25}{16}\left(\frac{1}{3}+\frac{1}{u}\right)\frac{(K-1)t}{\kappa m}\right).
\end{multline}

The minimum of the right hand side of the inequality  \eqref{eq:FirstGenBound} is reached for $u^*=~\frac{5}{4}\sqrt{\frac{(K-1)t}{\kappa m \rademacher_{T_\kappa(\TrnSet)}(\ell\circ\Fset_{r})}}$, plugging back the minimizer and solving for $\delta=e^{-t}$ gives the result.
\end{proof}

\bigskip

\begin{proof}[Proof of the theorem 1]
Theorem 1 of \cite{JoshiAPRUG15} states that fractional chromatic number of $T(\TrnSet)$ is bounded from above by $K-1$. Then by the lemma \ref{lem:MulticlassRedBinaryGenBound} with have with probability at least $1-\delta$~:
    \[
            R(h)\leq 
                \EmpRisk{T(\TrnSet)}(h) + \rademacher_{T(\TrnSet)}(\ell\circ\Fset_r) 
                +\frac{5}{2}\left(\sqrt{\rademacher_{T(\TrnSet)}(\ell\circ\Fset_r)}
                +\sqrt{\frac{r}{2}}\right)\sqrt{\frac{\log 1/\delta}{m}}
                +\frac{25}{48}\frac{\log 1/\delta}{m},
    \]
    entails the statement of the theorem 1. 
\end{proof}

\bigskip

\begin{customthm}{2 (a)}
\label{lem:SlowRatesAp}
 Let $\TrnSet=(\obs_i^{y_i})_{i=1}^m\in (\Input\times \Output)^m$ be a dataset of $m$ examples drawn i.i.d. according to a probability measure $\Dist = \Dist_\cY \times \Dist_{\Input | \cY}$ over $\Input\times \Output$ and $T(\TrnSet)$ the transformed set obtained with the transformation function $T$ defined in Eq.~(5). Let $\TrnSet_\pi \in (\Input\times \Output)^n$ and $T_\kappa(\TrnSet_\pi)$, $|T_\kappa(\TrnSet_\pi)| = M$ be a training sets derived from $\TrnSet$ and $T(\TrnSet)$ respectively using the algorithm $(\pi,\kappa)$-\mRb\ with parameters $\pi_1, \dots, \pi_K$ and $\kappa$. With the class of functions $\mathcal{G}=\{g: \Input \times \Output \rightarrow \mathbb{R} \}$ and $\mathcal H=\{h: h(\phi(\obs^y),\phi(\obs^{y'})) = g(\obs^{y})-g(\obs^{y'}), g\in\mathcal{G}\}$ we have the following bound on the expected risk of the classifier~:
\[
R(h) \leq 
    \alpha\EmpRisk{T_\kappa({\TrnSet}_\pi)}(h) + 
    \alpha \rademacher_{T_\kappa({\TrnSet}_\pi)}(\ell\circ \mathcal{H}) + 
    \alpha \sqrt{\frac{(K-1)\log 2/\delta}{2M\kappa}} + 
    \sqrt{\frac{2 \alpha\log 4K/\delta}{\beta(m-1)}} + 
    \frac{7\beta \log 4K/\delta}{3(m-1)}.
\]
holds with probability at least $1-\delta$, for any $\delta > 0$, the for all $h\in \mathcal \mathcal{H}$, 
$\ell:\{-1,+1\}\times\mathbb{R}\rightarrow [0,1]$ is the $0/1$ loss, and 
\[
    \alpha  = 
        \max_{y:\, 1\leq y \leq K} \eta_y/\pi_y, \quad
    \beta = 
        \max_{y:\, 1\leq y \leq K} 1/\pi_y,
\]
and $\eta_y$ is strictly positive empirical probability of the class $y$ over $\TrnSet$.
\end{customthm}
\begin{proof}
By lemma \ref{lem:main} we have for $\bar {\cal D} = \bar {\cal D}_\cY \times {\cal D}_{\Input|\cY}$, $\prob_{\bar {\cal D}_\cY} [y(\obs) = i] \propto \pi_i$, $1\le i \le K$ with probability at least $1-\delta/2$~:
\[
    R(h) \leq 
    \alpha\expectation_{\obs^y \sim \bar {\cal D}}
    \EmpRisk{T_\kappa({\TrnSet}_\pi)}(h) + 
    \sqrt{\frac{2 \alpha\log 4K/\delta}{\beta(m-1)}} + 
    \frac{7\beta \log 4K/\delta}{3(m-1)}.
\]

By theorem 4 of \cite{UsunierAG05} we have with probability at least $1-\delta/2$~:
\[
    \expectation_{\obs^y \sim \bar {\cal D}}\EmpRisk{T_\kappa(\TrnSet_\pi)}(h)
    \le 
    \EmpRisk{T_\kappa(\TrnSet_\pi)}(h) + 
    \rademacher(\mathcal{H}) + 
    \sqrt{\frac{\chi^*_{T_\kappa({\TrnSet}_\pi)}\log 2/\delta}{2M\kappa}},
\]
where a dependency graph for subsample $T_\kappa(\TrnSet_\pi)$ is a subgraph of the dependency graph for the whole sample~$T({\cal S})$. 

Then by lemma \ref{lem:chromatic} we have $\chi^*_{T_\kappa({\TrnSet}_\pi)} \leq \chi^*_{T(\TrnSet)} = K - 1$, the last is due to theorem 1 of \cite{JoshiAPRUG15}, where $\chi^*_{T(\underline{\TrnSet})}$ and $\chi^*_{T(\TrnSet)}$ stand for the fractional chromatic number of the dependency graph for $T_\kappa({\TrnSet}_\pi)$ and ${T(\TrnSet)}$ resp. 
Gather together the last two equations we prove the theorem.
\end{proof}

\bigskip

\bigskip
 
\begin{customthm}{2 (b)}
	\label{thm:FastRatesSubsetAp}
    Let $\TrnSet=(\obs_i^{y_i})_{i=1}^m\in (\Input\times \Output)^m$ be a dataset of $m$ examples drawn i.i.d. according to a probability distribution $\Dist$ over $\Input\times \Output$ and $T(\TrnSet)=((\bfZ_i,\tilde y_i))_{i=1}^N$ the transformed set obtained with the transformation function $T$ defined in Eq.~(5). 
    Let $\TrnSet_\pi \in (\Input\times \Output)^M$  and  $T_\kappa(\TrnSet_\pi)$ be a training set derived from $T(\TrnSet)$ using the algorithm $(\pi,\kappa)$-\mRb\ with parameters $\pi_1, \dots, \pi_K$ and $\kappa$.
    With the class of functions $\mathcal{G}=\{g: \Input \times \Output \rightarrow \mathbb{R} \}$ and $\mathcal H=\{h: h(\phi(\obs^y),\phi(\obs^{y'}))=g(\obs^{y})-g(\obs^{y'}), g\in\mathcal{G}\}$, consider the parameterized family $\Fset_{r}$ which, for $r>0$, is defined as~:
\[
\Fset_{r}=\{h:h\in\Fset,\mathbb{V}[h]\doteq\mathbb{V}_{\bfZ,\tilde{y}}[\Ind{\tilde{y}h(\bfZ)}]\leq r\},
\]
where $\mathbb{V}$ denotes the variance. Then for any $\delta>0$ with probability at least $(1-\delta)$ the following generalization bound holds for all $h\in \mathcal \Fset_{r}$~:
\begin{multline*}
    R(h) \leq 
        \alpha\EmpRisk{T_\kappa({\TrnSet}_\pi)}(h) + 
        \alpha \rademacher_{T_\kappa({\TrnSet}_\pi)}(\ell\circ \mathcal{H}_r) + 
        \alpha \sqrt{\frac{\log 4/\delta}{2m}} + 
        \sqrt{\frac{2 \alpha\log 4K/\delta}{\beta(m-1)}} + 
        \frac{7\beta \log 4K/\delta}{3(m-1)} 
        \\ +
        \frac{5\alpha}{2}\left(\sqrt{\rademacher_{T_\kappa({\TrnSet}_\pi)}(\ell\circ\Fset_r)} + 
        \sqrt{\frac{r}{2}}\right)\sqrt{\frac{(K-1)\log 2/\delta}{\kappa M}} + 
        \frac{25\alpha}{48}\frac{\log 2/\delta}{M},
\end{multline*}
where 
$\ell:\{-1,+1\}\times\mathbb{R}\rightarrow [0,1]$ is the $0/1$ loss and 
$\xi=(\xi_1,\ldots,\xi_N)$ a sequence of $N$ independent Rademacher variables such that $\proba(\xi_n=1)=\proba(\xi_n=-1)=1/2$,  and 
$\alpha  = \max_{y:\, 1\leq y \leq K} \eta_y/\pi_y,$ $\beta = \max_{y:\, 1\leq y \leq K} 1/\pi_y,$ and $\eta_y > 0$ is the empirical probability of the class $y$ over $\TrnSet$.
\end{customthm}

\bigskip

\begin{proof}
The proof of the theorem essentially combines the results of theorem 1 and lemma \ref{lem:main}. 

By lemma \ref{lem:main} we have with probability at least $1-\delta/2$~:
\begin{equation}
\label{eq:last1}
    R(h) \leq 
        \alpha\expectation_{\obs^y \sim \bar \Dist}\EmpRisk{T(\obs^y)}(h) + 
        \sqrt{\frac{2 \alpha\log 4K/\delta}{\beta(m-1)}} + 
        \frac{7\beta \log 4K/\delta}{3(m-1)}.
\end{equation}

Lemma \ref{lem:SlowRatesAp} applied to $T_\kappa({\TrnSet}_\pi)$, $T_\kappa({\TrnSet}_\pi) = M\kappa$ gives with probability at least $1-\delta/2$~: 
\begin{multline}
\label{eq:last2}
    \expectation_{\obs^y \sim \bar \Dist}\EmpRisk{T(\obs^y)}(h)\leq 
    \EmpRisk{T_\kappa(\TrnSet_\pi)}(h) + 
    \rademacher_{T_\kappa(\TrnSet_\pi)}(\ell\circ\Fset_r) + \\
    \frac{5}{2}\left(\sqrt{\rademacher_{T_\kappa(\TrnSet_\pi)}(\ell\circ\Fset_r)} + 
    \sqrt{\frac{r}{2}}\right)\sqrt{\frac{(K-1)\log 2/\delta}{M\kappa}} +
    \frac{25}{48}\frac{\log 2/\delta}{M}
\end{multline}

Substitution \eqref{eq:last1} in \eqref{eq:last2} gives~: 
\begin{multline*}	
	R(h)\leq 
	\alpha\EmpRisk{T_\kappa({\TrnSet}_\mu)}(h) + 
	\alpha \rademacher_{T_\kappa({\TrnSet}_\mu)}(\ell\circ \mathcal{H}_r) + 
	\sqrt{\frac{2 \alpha\log 4K/\delta}{\beta(m-1)}} +  
	\frac{7\beta \log 4K/\delta}{3(m-1)} + \\ 
	\frac{5\alpha}{2}\left(\sqrt{\rademacher_{T_\kappa({\TrnSet}_\mu)}(\ell\circ\Fset_r)} +
	\sqrt{\frac{r}{2}}\right)\sqrt{\frac{(K-1)\log 2/\delta}{M\kappa}} +
	\frac{25\alpha}{48}\frac{\log 2/\delta}{M}
\end{multline*}
\end{proof}

\begin{proof}[Proof of the theorem 2]
The statement of theorem 2 in the paper is essentially a union of the statements of theorem  \ref{lem:SlowRatesAp}  and theorem \ref{thm:FastRatesSubsetAp} proved above.
\end{proof}

\bigskip 

\section{Experimental Part}
\label{sec:app_exp}

Table \ref{tab:hyperpms} summarizes the parameters tuned for each of the methods. In our experiments, we used the solvers provided for each of the methods and tuned the most important parameters. The values shown in Table \ref{tab:hyperpms} are the final values used in the  experimental results reported in the paper, which resulted in best predictive performance in the heldout dataset.  

\begin{table}[!htb]
    \resizebox{\textwidth}{!}{
    \begin{tabular}{|c|c|c|c|c|c|c|}
        \hline
       Algorithm & Parameters  & LSHTC1 & DMOZ & WIKI-Small & WIKI-50K & WIKI-100K \\ \hline
       $\OVA$ &  C & 10  & 10 & 1 & NA & NA  \\ \hline
       $\MSVM$ &  C & 1 & 1& NA & NA & NA \\ \hline
       $\RecallTree$ & b & 30 & 30 & 30& 30 & 28 \\ \cline{2-7}
        & l   &1 &0.7 & 0.7& 0.5& 0.5\\ \cline{2-7}
        & loss\_function  & Hinge& Hinge& Logistic &Hinge &Hinge \\ \cline{2-7}
        & passes & 5 & 5& 5& 5& 5 \\ \hline
        $\fastxml$ & t & 100 & 50 & 50 & 100 &50 \\ \cline{2-7}
                   & c & 100 & 100 & 10 & 10 & 10\\ \hline
         $\pfastrexml$ & t &50  &50 & 100&200 &100 \\ \cline{2-7}
                   & c & 100 & 100 & 10 &10 &10 \\ \hline
        $\pdsparse$ &  l & 0.01 & 0.01& 0.001& 0.0001& 0.01\\ \cline{2-7}
        & Hashing & multiTrainHash & multiTrainHash & multiTrainHash & multiTrainHash & multiTrainHash\\ \hline
        \mRb & Examples per class in average$\textsuperscript{*}$& 2 & 2 & 1  & 1  & 1 \\ \cline{2-7}
         & Adversarial Classes ($\kappa$)        & 122  & 27 & 36 & 5 & 10  \\ \cline{2-7}
         & Candidate Classes  ($q$)                   & 10 & 10& 10& 10& 10\\ \hline
        \multicolumn{7}{l}{\textsuperscript{*}\footnotesize{ Here examples per class for proposed $\mRb$ method represents the average number of examples sampled per class.
        The examples are chosen at random }} \\
        \multicolumn{7}{l}{\footnotesize{from each class with probability $\pi_k$ based on the distribution.}}

    \end{tabular}
    }
    \caption{Hyper-parameters used in the final experiments}
    \label{tab:hyperpms}
    
\end{table}

\end{document}